\documentclass[11pt]{article}
\usepackage{xcolor}
\usepackage[colorlinks]{hyperref}
\hypersetup{
	linkcolor= blue
	,citecolor= green
	,filecolor= cyan
	,urlcolor= magenta
	,menucolor= red
	,runcolor= cyan
}
\usepackage{float}
\usepackage{authblk}
\usepackage[margin=1in]{geometry}
\usepackage{graphicx} 
\usepackage{amsmath}
\usepackage{bm}
\usepackage{tikz, pgfplots}
\usetikzlibrary{arrows.meta, positioning, shapes.geometric, calc}
\usepackage{accents}
\usepackage{amssymb}
\usepackage{amsthm}
\usepackage{booktabs}
\usepackage{cleveref}
\usepackage{}
\theoremstyle{plain}
\newtheorem{theorem}{Theorem}[section] 

\theoremstyle{definition}
\newtheorem{definition}{Definition}[section]

\newcommand{\eqautoref}[1]{\hyperref[#1]{Eq.~(\ref*{#1})}}
\definecolor{inputcolor}{RGB}{102,178,255}      
\definecolor{embcolor}{RGB}{255,204,102}        
\definecolor{enccolor}{RGB}{153,204,153}        
\definecolor{sdecolor}{RGB}{255,153,153}        
\definecolor{latcolor}{RGB}{204,153,255}        
\definecolor{deccolor}{RGB}{255,204,229}        
\definecolor{outcolor}{RGB}{192,192,192}        

\title{\bf{Modeling Social Indicator Dynamics with Variational Neural Stochastic Differential Equations}}

\author[1,2]{Sandeep Kumar Samota\thanks{\textcolor{magenta}{\texttt{sandeepksamota@gmail.com}}}}
\author[1]{Reema Gupta\thanks{\textcolor{magenta}{\texttt{reema270898@gmail.com}}}}
\author[1,2]{S. Chakraverty\thanks{\textcolor{magenta}{\texttt{Corresponding author:sne\_chak@yahoo.com}}}}
\affil[1]{Department of Mathematics, National Institute of Technology, Rourkela, India}
\affil[2]{Poverty Alleviation Research Centre, National Institute of Technology, Rourkela, India}
\date{}
\begin{document}
	\maketitle

	\begin{abstract}
		This study examines the challenges of modeling complex and noisy data related to socioeconomic factors over time, with a focus on data from various districts in Odisha, India. Traditional time-series models struggle to capture both trends and variations together in this type of data. To tackle this, a Variational Neural Stochastic Differential Equation (V-NSDE) model is designed that combines the expressive dynamics of Neural SDEs with the generative capabilities of Variational Autoencoders (VAEs). This model uses an encoder and a decoder. The encoder takes the initial observations and district embeddings and translates them into a Gaussian distribution, which determines the mean and log-variance of the first latent state. Then the obtained latent state initiates the Neural SDE, which utilize neural networks to determine the drift and diffusion functions that govern continuous-time latent dynamics. These governing functions depend on the time index, latent state, and district embedding, which help the model learn the unique characteristics specific to each district. After that, using a probabilistic decoder, the observations are reconstructed from the latent trajectory. The decoder outputs a mean and log-variance for each time step, which follows the Gaussian likelihood. The Evidence Lower Bound (ELBO) training loss improves by adding a KL-divergence regularization term to the negative log-likelihood (nll). The obtained results demonstrate the effective learning of V-NSDE in recognizing complex patterns over time, yielding realistic outcomes that include clear trends and random fluctuations across different areas. 
	\end{abstract}
	\section{Introduction}
Socioeconomic development over time is a complex, dynamic process. The structural trends, short-term policies, and uncertain shocks can often impact this process. In this context, socioeconomic indicators that affect household poverty are continually evolving and are becoming increasingly diverse across regions. Due to some geographical, administrative, infrastructural, and historical factors, these indicators pose nonlinear trajectories, fluctuations, and disparities. In socioeconomically diverse states like Odisha, India, capturing such dynamics in a principled and uncertainty-aware manner remains a central challenge \cite{panda2021poverty}.\\

Traditional approaches to model such indicators mostly rely on discrete-time econometric and statistical methods such as regression, panel data, and autoregressive techniques \cite{kessels2016structural,odjesa2021socio, kumar2023multidimensional}. For example, Dang and Lanjouw \cite{dang2018poverty} analyzed poverty dynamics in India by constructing the panel dataset from various surveys. Cheng et al. \cite{cheng2018building} employed a system dynamics model, a simulation model used to understand complex phenomena. They used it to explore the relationship between environment, disasters and poverty. Łatuszyńska and Fate \cite{latuszynska2022combining} presented a hybrid model combining system dynamics and agent-based simulation to study the impact of public intervention on poverty. More recently, Kumar et al. \cite{kumar2025assessing} constructed a wealth index using principal component analysis for a district in Odisha, utilizing primary data collected, and compared the outcomes with a similar index from the national family health survey (NFHS) surveys.\\

These methods impose strong parametric assumptions, depend on linear or weak nonlinear relationships, and operate at fixed observation intervals. Moreover, Indian household data obtained from surveys like the district-level household and facility survey (DLHS) and NFHS are often found to be sparse, irregularly sampled, and noisy. It limits the effectiveness of traditional time-series models. Furthermore, these limitations become more evident when modeling multiple indicators with an interacting nature over a long time, where uncertainty and latent influences are crucial.\\

In this regard, stochastic differential equations (SDEs) can provide a mathematical framework for modeling such systems \cite{ross1995stochastic}. They can decompose the problem into two parts: deterministic and stochastic. By doing this, SDEs can model both drift trends and the stochastic nature of the socioeconomic problem. Therefore, it can be said that SDEs allow continuous-time modeling, uncertainty quantification, and theoretical interpretability. However, classical SDE models require the drift and diffusion functions to be specified in advance using some simplistic parametric forms that may be inadequate for capturing the high-dimensional and nonlinear interaction characteristics of real-world socioeconomic systems \cite{gupta2025spectral}. Furthermore, parameter estimation in multivariate SDEs can quickly become intractable as model complexity increases.\\

The developments in deep learning (DL) techniques have led to the advancement of continuous-time dynamical systems governed by neural networks \cite{sahoo2024unsupervised}. Beginning with neural ordinary differential equations (Neural ODEs), which parameterize deterministic continuous-time dynamics using neural networks and enable end-to-end training through differentiable ODE solvers \cite{chen2018neural}. This model establishes a conceptual bridge between DL and dynamical systems theory. It provides a principled approach to modeling nonlinear temporal dynamics without introducing any discretization. Also, parallel developments such as neural controlled differential equations (Neural CDEs), which basically treat observed time series as control signals \cite{kidger2020neural}. These drive latent continuous-time dynamics and are particularly effective for irregularly sampled data.\\

Building upon this foundation, neural stochastic differential equations (Neural SDEs) extend Neural DEs-based methodologies by incorporating stochasticity into the dynamics through learned diffusion terms \cite{liu2019neural}. Neural SDEs can effectively model the randomness, shocks, and variability present in socioeconomic dynamics \cite{jia2019neural}. They utilize the expressive power of NNs with the mathematical rigor of stochastic calculus, enabling data-driven discovery of complex continuous-time dynamics. More recent work focuses on scalable and amortized variational inference for latent SDEs, addressing computational and statistical challenges associated with learning continuous-time stochastic models from sparse and noisy observations \cite{li2019scalable,li2020scalable}. Neural SDE models have demonstrated remarkable success in modelling applications related to finance, biology, and physics, where systems are governed by stochastic processes \cite{oh2025comprehensive}. Despite all, the application of neural SDEs in modeling socioeconomic dynamics remains limited and underexplored.\\

A major challenge in learning neural SDEs from real data is the intractability of the likelihood function. Variational inference offers a solution to this problem by approximating the intractable posterior distribution over latent trajectories with a tractable variational family. Variational neural stochastic differential equations (V-NSDEs) enable efficient learning of latent continuous-time processes and provide uncertainty-aware predictions \cite{tzen2019neural}. When combined with latent variable formulations, variational inference models allow high-dimensional observed indicators to be governed by a lower-dimensional latent stochastic process.\\
	
Here, we will utilize district-level socioeconomic data related to poverty in the state of Odisha, India. Odisha exhibits substantial inter-district heterogeneity across all socioeconomic indicators, such as education, housing quality, sanitation access, etc. Historically marginalized regions like the KBK (Kalahandi–Balangir–Koraput) have different developmental trajectories compared to coastal and other urbanized districts. Available survey measurements of these indicators are available at multi-year intervals and are subject to compounding uncertainty. These studies motivate the need to develop a modeling framework that can (i) explicitly handle irregular temporal sampling, (ii) quantify predictive uncertainty, and (iii) provide interpretable latent structure properties that variational latent SDEs naturally offer. A V-NSDE framework can allow such observations to be embedded within a continuous-time stochastic model. It can also reconstruct the latent trajectory and principal uncertainty quantification at the district level.\\
	
In addition to temporal uncertainty and nonlinear dynamics, district-level socioeconomic modeling should understand and incorporate spatial heterogeneity, institutional capacity, historical deprivation, and infrastructural endowments. To address this, district embedding vectors can be incorporated into the V-NSDE framework. These vectors can be learned during training the model and can act as a condition on the drift and diffusion functions. These embeddings act as continuous representations of district-specific characteristics, which enables the model to share statistical strength across districts while simultaneously capturing systematic inter-district differences.\\

In this work, we advance the V-NSDE framework for modeling socioeconomic dynamics by developing a latent continuous-time formulation. This method incorporates district-specific embeddings for multivariate district-level data. Our approach learns nonlinear drift and diffusion functions that govern the evolution of latent socioeconomic states, while explicitly accounting for observational noise and irregular sampling. The proposed model can provide a reliable tool for analyzing long-term socioeconomic processes under uncertainty.\\
	
We will validate the proposed method using district-level socioeconomic indicators from Odisha. Further, we will demonstrate that V-NSDE can effectively capture heterogeneous developmental patterns across districts. Beyond the specific case of Odisha, the proposed framework offers a general methodological foundation for continuous-time modeling of socioeconomic systems in data-constrained and uncertain environments.\\
	
The remainder of the paper is organized as follows. Section 2 presents the mathematical formulation of neural stochastic differential equations and the associated variational inference framework. Section 3 describes the proposed latent model structure and learning methodology. Section 4 details the district-level Odisha dataset and experimental setup. Section 5 discusses the empirical results and their socioeconomic implications. Finally, Section 6 concludes the paper and outlines directions for future research.\\
	
%
%

\section{Mathematical Preliminaries}
This section provides the fundamental concepts required for the formulation of V-NSDE.

\subsection{Probability Space and Random Variables}
\begin{definition}[\cite{rohatgi2015introduction}]
	If $\Omega$ denotes the sample space, $\mathcal{F}$ is a $\sigma$-algebra of measurable events, and $\mathbb{P}$ is a probability measure satisfying $\mathbb{P}(\Omega)=1$, then triplet $(\Omega, \mathcal{F}, \mathbb{P})$ is called \textbf{probability space}.
\end{definition}

\begin{definition}[\cite{rohatgi2015introduction}]
	If $(\Omega, \mathcal{F}, \mathbb{P})$ is probability space, then the map \[
	X : \Omega \rightarrow \mathbb{R}^m.
	\] is called \textbf{random variable}.
\end{definition}

The probability distribution of $X$ can be characterized by probability mass function (pmf) in the discrete case and probability density function (pdf) in the continuous case. Now, we will define expectation of random variable $X$.
\begin{definition}[\cite{rohatgi2015introduction}]
	The \textbf{expectation} of a random variable $X$ is defined as
	\[
	\mathbb{E}[X] =
	\begin{cases}
		\sum_{x} x \, \mathbb{P}(X=x), & \text{discrete case}, \\[6pt]
		\int_{\mathbb{R}} x \, p(x)\, dx, & \text{continuous case},
	\end{cases}
	\] where $p(x)$ is pdf of $X$.
\end{definition}
More generally, the expectation of a function $g(X)$ is defined as
\[
\mathbb{E}_{p(x)}[g(X)] = \int g(x)\, p(x)\, dx.
\]

It follows linearity condition, 
\[
\mathbb{E}[aX + bY] = a\,\mathbb{E}[X] + b\,\mathbb{E}[Y],
\]
for any random variables $X, Y$ and constants $a,b \in \mathbb{R}$.

\begin{definition}[\cite{rohatgi2015introduction}]
	The \textbf{variance} of a random variable $X$ is defined as
	\[
	\mathrm{Var}(X) = \mathbb{E}\left[(X - \mathbb{E}[X])^2\right].
	\]
\end{definition}
The variance measures the dispersion around the mean.
\begin{definition}[\cite{rohatgi2015introduction}]
	If $X$ and $Y$ are random variables, then \textbf{conditional probability} of $X$ given $Y=y$ is defined as
	\[
	p(x \mid y) = \frac{p(x,y)}{p(y)}, \quad p(y)>0. 
	\]
\end{definition}

\subsection{Bayes' Rule}
Bayes' rule provides a fundamental relationship between prior knowledge, observed data, and posterior inference. It forms the mathematical backbone of Bayesian statistics and variational inference.\\

Let $x$ be a observed data and $z$ an unobserved (latent) random variable. Then, the joint distribution of $(x,z)$ is factorized as
\[
p(x,z) = p(x \mid z)\, p(z),
\]
here $p(z)$ is the prior distribution over the latent variable and $p(x \mid z)$ is the likelihood function.\\

Bayes' rule expresses the posterior distribution $p(z \mid x)$ as \cite{rohatgi2015introduction}
\[
p(z \mid x)=\frac{p(x \mid z)\, p(z)}{p(x)}
\]
where the denominator
\[
p(x) = \int p(x \mid z)\, p(z)\, dz
\]
is known as the marginal likelihood or evidence, which is often intractable for complex models.

\begin{definition}[\cite{kullback1951information}]
	The \textbf{KL divergence} between two probability distributions $p(x) \sim P$ and $q(x) \sim Q$ is defined as
	\[
	D_{KL}(P \| Q)
	=
	\int p(x)\, \log \frac{p(x)}{q(x)}\, dx.
	\]
\end{definition}
It satisfies $\mathrm{KL}(P\|Q)\ge 0$ and is zero if and only if $P=Q$ almost everywhere. It measures similarity between the probability distributions.

\subsection{Monte Carlo Approximation}
Some expectations are analytically intractable. For those Monte Carlo sampling approximation are done as
\[
\mathbb{E}_{p(X)}[g(X)]
\approx
\frac{1}{N}\sum_{i=1}^N g(x_i),
\quad x_i \sim p(X).
\]
Monte Carlo estimation plays a central role in training probabilistic deep learning models like our which includes VAEs.

\subsection{Brownian Motion}
\begin{definition}[\cite{oksendal2013stochastic}]
	A \textbf{stochastic process} is a time parameterized collection of random variables
	\[\{X_t\}_{t \ge 0},\]
	defined on probability space $(\Omega, \mathcal{F}, \mathbb{P})$.
\end{definition}
\begin{definition}[\cite{oksendal2013stochastic}]
	A $d$-dimensional stochastic process $\{W_t\}_{t \ge 0}$ is called standard \textbf{Brownian motion} (or Wiener process), if it satisfies
	\begin{enumerate}
		\item $W_0 = 0$ almost surely,
		\item $W_t$ has independent increments,
		\item $W_t - W_s \sim \mathcal{N}(0, (t-s)I_d)$ for $0 \le s < t$,
		\item $W_t$ has almost surely continuous paths.
	\end{enumerate}
\end{definition}

\subsection{It\^{o} SDEs}

For $X_t \in \mathbb{R}^m$, It\^{o} SDE is defined as \cite{oksendal2013stochastic}
\[
dX_t = f(X_t,t)\,dt + g(X_t,t)\,dW_t,
\quad X_0 = x_0,
\]
where $f$ is the drift function, $g$ is the diffusion function, and $W_t$ is an $m$-dimensional Brownian motion. In the SDE, the drift term models deterministic dynamics and the diffusion term captures stochastic fluctuations.\\

The SDE can be equivalently written in integral form as
\[
X_t = X_0 + \int_0^t f(X_s,s)\,ds + \int_0^t g(X_s,s)\,dW_s,
\]
where the stochastic integral is interpreted in the It\^{o} sense.

\section{V-NSDE Methodology}
This section presents the essential methodological components of the V-NSDE used in the current study. Firstly, the detailed forward architecture is explained whihc emphasizes the structure and functionality of the model. Subsequently, the loss formulation is discussed in terms of its importance and influence on model performance. Finally, the process of backward propagation is discussed, illustrating how it optimizes the learning mechanism and enhances the overall effectiveness of the model.
\subsection{Forward Architecture Components}
The forward part of the model is consists of four components: Encoder, Latent NSDE, and Decoder. In the subsequent subsections each one is systematically described. The schematic diagram of the V-NSDE method is presented in Fig. (\ref{fig:latent-sde-architecture-colored}).
\begin{figure}[ht]
	\centering
	\begin{tikzpicture}[
		>=latex,
		node distance=1.8cm,
		block/.style={draw, thick, rounded corners, align=center,
			minimum height=1.2cm, minimum width=2.8cm},
		smallblock/.style={draw, thick, rounded corners, align=center,
			minimum height=1cm, minimum width=2.0cm},
		every node/.style={font=\small}
		]
		
		
		\node[block, fill=inputcolor!60] (yseq) {Input\\ $\mathbf{y}_0\in \mathbb{R}^N$};
		\node[smallblock, below=1.4cm of yseq, fill=inputcolor!40] (did) {District ID\\ $d$};
		
		\node[block, right=1.4cm of did, fill=embcolor!70] (emb) {Embedding\\ $\mathbf{e}_d \in \mathbb{R}^m$};
		
		\node[block, right=1.0cm of yseq, fill=enccolor!70] (enc) {Encoder \\ $q_{\phi}(\mathbf{z}_0 | \mathbf{y}_0, \mathbf{e}_{d})$};
		\node[block, above =1.5cm of enc, fill=latcolor!70] (z0) {Initial latent state\\ $\mathbf{z}_0 \in \mathbb{R}^n$};
		
		\node[block, right=1.0cm of z0, minimum width=3.5cm, fill=sdecolor!70] (sde) {Latent SDE\\[2pt]
		$d\mathbf{z}_t = f_\theta(\mathbf{z}_t, t, \mathbf{e}_{d})dt + g_\theta(\mathbf{z}_t, t, \mathbf{e}_{d})d\mathbf{W}_t$};
		\node[block, below = 1.0cm of sde, fill=latcolor!70] (sdesolver) {Numerical\\
		SDE solver};
		\node[block, below = 1.0cm of sdesolver, fill=latcolor!70] (ztraj) {Latent trajectory\\
		$\mathbf{z}_0,\dots,\mathbf{z}_{T-1}$};
		
		\node[block, below = 1.0cm of ztraj, fill=deccolor!80] (dec) {Decoder\\ $p_{\theta}(\mathbf{y}_t | \mathbf{z}_t, \mathbf{e}_{d})$};
		
		\node[block, below = 1.0cm of dec, fill=outcolor!60] (yhat) {Predicted outputs\\
		$\boldsymbol{\mu}_{y_t}, \boldsymbol{\sigma}^2_{y_t}$};
		
		\draw[->, thick] (yseq) -- (enc);
		\draw[->, thick] (enc) -- (z0);
		\draw[->, thick] (z0) -- (sde);
		\draw[->, thick] (sde) -- (sdesolver);
		\draw[->, thick] (sdesolver)--(ztraj);
		\draw[->, thick] (ztraj) -- (dec);
		\draw[->, thick] (dec) -- (yhat);
		
		\draw[->, thick] (did) -- (emb);
		
		\draw[->, thick] (emb) --  (enc);
		
		\draw[->, thick] (emb) -- (sde);
		
		\draw[->, thick] (emb) -- (dec);
		
		\node[above = 0.1cm of sde] (time) {$t = 0,1, \dots,T-1$};
		
	\end{tikzpicture}
	\caption{Block diagram of the proposed V-NSDE architecture for district $d$.}
	\label{fig:latent-sde-architecture-colored}
\end{figure}
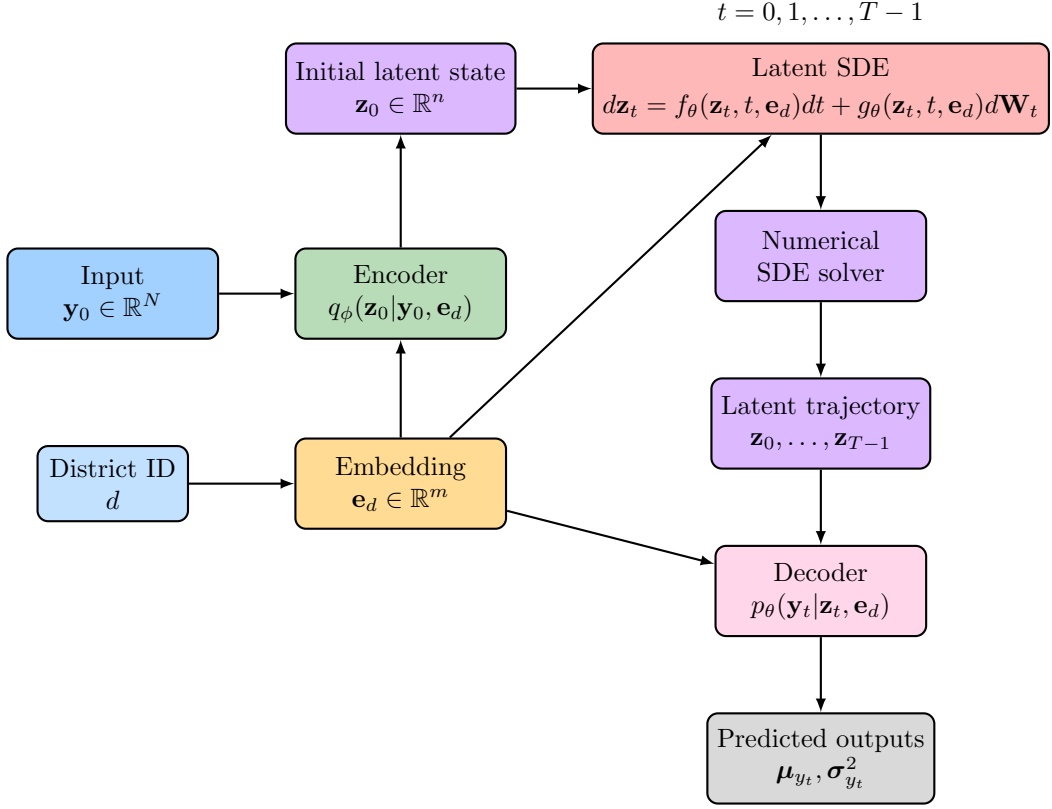
\subsubsection{District Embedding}
To capture the district specific characteristics, embedding technique is utilized. It converts discrete inputs into low-dimensional numerical vectors. In this regard, each district is associated with a learnable, fixed-size embedding vector $\mathbf{e}_{d} \in \mathbb{R}^{m}$. These embeddings serve to personalize the VAE components for each district, allowing the model to capture district-specific characteristics that influence the observed socio-economic indicators and their temporal dynamics. 
\subsubsection{Encoder: $q_{\phi}(\mathbf{z}_0 | \mathbf{y}_0, \mathbf{e}_{d})$}
The encoder is a deep neural network (DNN) whose primary role is to establish the initial probabilistic latent state $\mathbf{z}_0 \in \mathbb{R}^n$ from the very first observation of a time series $\mathbf{y}_0 \in \mathbb{R}^N$ and its corresponding static contextual information, the district embedding $\mathbf{e}_{d} \in \mathbb{R}^m$. As given in Eq. (\ref{eq:posterior distribution phi}), it models the approximate posterior distribution $q_\phi(\mathbf{z}_0 | \mathbf{y}_0, \mathbf{e}_{d})$ as a diagonal Gaussian distribution over $\mathbf{z}_0$.
\begin{equation}\label{eq:posterior distribution phi}
	q_{\phi}(\mathbf{z}_0 | \mathbf{y}_0, \mathbf{e}_{d}) = \mathcal{N}(\boldsymbol{\mu}_{z_0}, \text{diag}(\exp(\log \boldsymbol{\sigma}^2_{z_0})))
\end{equation}
	
The Encoder takes the concatenated initial observation and district embedding as input and output the mean $\boldsymbol{\mu}_{z_0}$ and log-variance $\log \boldsymbol{\sigma}^2_{z_0}$ as described in Eq. (\ref{eq:encoder}).
\begin{equation}\label{eq:encoder}
	\begin{aligned}
		[\boldsymbol{\mu}_{z_0}, \log \boldsymbol{\sigma}^2_{z_0}] &= \text{Encoder}([\mathbf{y}_0; \mathbf{e}_{d}]), \\ 
		\mathbf{z}_0 &= \boldsymbol{\mu}_{z_0} + \boldsymbol{\sigma}_{z_0} \odot \boldsymbol{\epsilon}, \quad \boldsymbol{\epsilon} \sim \mathcal{N}(\mathbf{0}, \mathbf{I}),
	\end{aligned}
\end{equation}
where $\odot$ denotes element-wise multiplication and $\boldsymbol{\epsilon}$ is a sample from a standard normal distribution for reparameterization trick.\\
	
The Encoder compresses the initial high-dimensional observation into a lower-dimensional, probabilistic latent representation that captures its essential features. The sampled $\mathbf{z}_0$ then serves as the starting point for the subsequent SDE trajectory. The probabilistic nature allows the model to capture uncertainty in the initial state.

\subsubsection{Latent NSDE}
The latent NSDE models the continuous-time evolution of the latent state $\mathbf{z}_t$ given its initial value $\mathbf{z}_0$ (sampled from the encoder) and the static district embedding $\mathbf{e}_{d}$. It captures both deterministic dynamics and stochastic fluctuations within the latent space. The temporal dynamics of the latent state $\mathbf{z}_t \in \mathbb{R}^{n}$ are governed by an Ito SDE defined in Eq. (\ref{eq:VNSDE}), conditioned on time $t$ and the district embedding $\mathbf{e}_{d}$.
\begin{equation}\label{eq:VNSDE}
	d\mathbf{z}_t = f_\theta(\mathbf{z}_t, t, \mathbf{e}_{d})dt + g_\theta(\mathbf{z}_t, t, \mathbf{e}_{d})d\mathbf{W}_t.
\end{equation}
Here, $f_\theta: \mathbb{R}^{n} \times \mathbb{R} \times \mathbb{R}^{m} \to \mathbb{R}^{n}$ is the drift function and $g_\theta: \mathbb{R}^{n} \times \mathbb{R} \times \mathbb{R}^{m} \to \mathbb{R}^{n}$ is the diagonal diffusion function. Both $f_\theta$ and $g_\theta$ are parameterized by DNNs. $\mathbf{W}_t$ represents a standard Brownian motion. In a diagonal noise SDE, each dimension of the latent space has its own independent Wiener process, meaning $\mathbf{W}_t$ is a vector of independent random increments.\\
	
We assume that the neural drift and diffusion functions, augmented with district embeddings, satisfy global Lipschitz and linear growth conditions. Under these assumptions, the latent NSDE admits a unique strong solution for each district.\\
	
Given the initial latent state $\mathbf{z}_0$ and the parameterized drift and diffusion functions, the latent SDE is numerically integrated over the time interval $[0, T]$ using an SDE solver to get the solution given as in Eq. \eqref{eq:sde solver}. This process yields the full latent trajectory $\{\mathbf{z}_t\}_{t=0}^T$.
	
\begin{equation}\label{eq:sde solver}
	\mathbf{z}_t = \mathbf{z}_0 + \int_{0}^{t} f_\theta(\mathbf{z}_s, s, \mathbf{e}_{d})ds + \int_{0}^{t} g_\theta(\mathbf{z}_s, s, \mathbf{e}_{d})d\mathbf{W}_s
\end{equation}
	
The Latent SDE is the heart of the temporal modeling. It allows the model to learn complex, non-linear, and stochastic dependencies between successive latent states, ensuring that the generated latent trajectories are plausible and reflect the underlying dynamics of the real-world time series. The inclusion of $\mathbf{e}_{d}$ allows the dynamics to be specific to each district.
\subsubsection{Decoder: $p_{\theta}(\mathbf{y}_t | \mathbf{z}_t, \mathbf{e}_{d})$}
The Decoder is also a DNN model that acts as the generative network, taking the evolved latent state $\mathbf{z}_t$ from the NSDE and the district embedding $\mathbf{e}_{d}$ to reconstruct the observed data $\mathbf{y}_t$. As described in Eq. (\ref{eq:likelihood distribution theta}), it models the reconstruction output by the likelihood $p_{\theta}(\mathbf{y}_t | \mathbf{z}_t, \mathbf{e}_{d})$ as a diagonal Gaussian distribution.
	\begin{equation}\label{eq:likelihood distribution theta}
		p_{\theta}(\mathbf{y}_t | \mathbf{z}_t, \mathbf{e}_{d}) = \mathcal{N}(\mathbf{y}_t; \boldsymbol{\mu}_{y_t}, \text{diag}(\exp(\log \boldsymbol{\sigma}^2_{y_t})))
	\end{equation}
	The Decoder outputs the mean $\mu_{y_t}$ and log-variance $\log \sigma^2_{y_t}$ that takes the latent state and district embedding as input as given in Eq. (\ref{eq:decoder}).
	\begin{equation}\label{eq:decoder}
		[\boldsymbol{\mu}_{y_t}, \log \boldsymbol{\sigma}^2_{y_t}] = \text{Decoder}([\mathbf{z}_t; \mathbf{e}_{d}])
	\end{equation}
	 For each latent state $\mathbf{z}_t$ along the trajectory, the decoder computes the parameters of a Gaussian distribution over the observed features $\mathbf{y}_t$. The inclusion of $\mathbf{e}_{d}$ allows the reconstruction to be contextualized by the specific district characteristics.
	\subsection{Backward Model Architecture Components}\label{sec. bacward model}
	\subsubsection{Evidence Lower Bound (ELBO) Loss Function Derivation}
	The primary objective when training the V-NSDE model is to maximize the marginal likelihood of the observed data, $p_\theta(\mathbf{Y})$, where $\mathbf{Y} = \{\mathbf{y}_t\}_{t=0}^T$ is the entire observed time series for a given district. The marginal likelihood $p_\theta(\mathbf{Y})$ can be written by integrating over the latent trajectory $\mathbf{Z} = \{\mathbf{z}_t\}_{t=0}^T$:
	\begin{equation}\label{eq:marginal likelihood}
		p_\theta(\mathbf{Y}) = \int p_\theta(\mathbf{Y}|\mathbf{Z}) p_\theta(\mathbf{Z}) d\mathbf{Z}.
	\end{equation}
	
	However, directly maximizing $p_\theta(\mathbf{Y})$ is intractable due to the complex NN model. Instead, we optimize a lower bound on the marginal likelihood, known as the Evidence Lower Bound (ELBO). The ELBO can be derived using Jensen's inequality and properties of logarithms.\\
	
	We introduce the approximate posterior $q_{\phi}(\mathbf{Z}|\mathbf{Y})$, which is approximated by $q_{\phi}(\mathbf{z}_0|\mathbf{y}_0, \mathbf{e}_{d})$ and the SDE dynamics \cite{hershey2007approximating}. Therefore, Eq. (\ref{eq:marginal likelihood}) becomes
	\begin{equation}\label{eq:marginal likelihood_next}
		p_\theta(\mathbf{Y}) = \int q_{\phi}(\mathbf{Z}|\mathbf{Y}) \frac{p_\theta(\mathbf{Y}|\mathbf{Z}) p_\theta(\mathbf{Z})}{q_{\phi}(\mathbf{Z}|\mathbf{Y})} d\mathbf{Z}.
	\end{equation}
	
	Now, taking the logarithm and applying Jensen's inequality to Eq. (\ref{eq:marginal likelihood_next}) as $\log$ is a concave function, the following inequality can be obtained:
	\begin{equation}\label{eq:ELBO inequality}
		\log p_\theta(\mathbf{Y}) \ge \mathbb{E}_{q_{\phi}(\mathbf{Z}|\mathbf{Y})} \left[ \log \frac{p_\theta(\mathbf{Y}|\mathbf{Z}) p_\theta(\mathbf{Z})}{q_{\phi}(\mathbf{Z}|\mathbf{Y})} \right].
	\end{equation}
	
	This lower bound in Eq. (\ref{eq:ELBO inequality}) is the ELBO, which further can be simplified as:
	\begin{equation}\label{eq:elbo loss open}
		\text{ELBO}(\phi, \theta; \mathbf{Y}) = \mathbb{E}_{q_{\phi}(\mathbf{Z}|\mathbf{Y})} [\log p_\theta(\mathbf{Y}|\mathbf{Z})] - \mathbb{E}_{q_{\phi}(\mathbf{Z}|\mathbf{Y})} [\log \frac{q_{\phi}(\mathbf{Z}|\mathbf{Y})}{p_\theta(\mathbf{Z})}].
	\end{equation}
	
	Recognizing the second term in Eq. (\ref{eq:elbo loss open}) as the KL divergence, $D_{KL}(q_{\phi}(\mathbf{Z}|\mathbf{Y}) || p_\theta(\mathbf{Z}))$, the ELBO can be written as:
	\begin{equation}\label{eq:elbo}
		\text{ELBO}(\phi, \theta; \mathbf{Y}) = \mathbb{E}_{q_{\phi}(\mathbf{Z}|\mathbf{Y})} [\log p_\theta(\mathbf{Y}|\mathbf{Z})] - D_{KL}(q_{\phi}(\mathbf{Z}|\mathbf{Y}) || p_\theta(\mathbf{Z})).
	\end{equation}
	According to Eq. (\ref{eq:ELBO inequality}), maximizing ELBO is equivalent to maximize the likelihood $p_\theta(\mathbf{Y})$. Further, it is equivalent to minimizing the negative ELBO during training, given by:
	\begin{equation}\label{eq:final loss analytic}
		\mathcal{L}(\phi, \theta; \mathbf{Y}) = -\text{ELBO}(\phi, \theta; \mathbf{Y}) = \underbrace{-\mathbb{E}_{q_{\phi}(\mathbf{Z}|\mathbf{Y})} [\log p_\theta(\mathbf{Y}|\mathbf{Z})]}_{\text{Reconstruction Loss (NLL)}} + \beta \underbrace{D_{KL}(q_{\phi}(\mathbf{Z}|\mathbf{Y}) || p_\theta(\mathbf{Z}))}_{\text{KL Divergence Term}}.
	\end{equation}
	
Here, reconstruction loss term measures how well the decoder can reconstruct the observed data $\mathbf{Y}$ from the sampled latent trajectory $\mathbf{Z}$. In Gaussian likelihood setting, this typically translates to minimizing a form of Mean Squared Error or, more accurately, the Gaussian Negative Log-Likelihood (NLL). $\beta$ represents beta annealing, which is a technique recommended to prevent ``posterior collapse" by gradually increasing the weight of the KL divergence term during training.\\
	
The KL divergence term measures the difference between the approximate posterior distribution $q_{\phi}(\mathbf{Z}|\mathbf{Y})$, (the distribution over latent states inferred by the encoder and SDE dynamics) and the prior distribution $p(\mathbf{Z})$ (the assumed distribution over latent states, typically a simple one like a standard normal distribution). Minimizing this term encourages the approximate posterior to stay close to the prior, preventing the latent space from becoming too complex or irregular.
	
\subsubsection{Weight Updates with Adjoint Method}
Training a deep learning model involves adjusting its parameters (weights and biases) to minimize a loss function. This is typically achieved using gradient-based optimization techniques like Stochastic Gradient Descent (SGD) or Adam, which require computing the gradients of the loss with respect to all model parameters via backpropagation.\\
	
In the V-NSDE model, the latent trajectory $\mathbf{Z} = \{\mathbf{z}_t\}_{t=0}^T$ is generated by numerically integrating an SDE. Standard backpropagation through reverse-mode automatic differentiation (AD) would need to store the entire history of intermediate computations (activations and gradients) from the SDE solver at each time step. For SDEs simulated over many time steps, this can lead to an enormous memory footprint and computational cost, making training impractical.\\

In this regard, the adjoint sensitivity method provides an efficient way to compute gradients through DE solvers without storing all intermediate states. Instead of backpropagating through each step of the SDE solver, it solves a coupled backward-in-time DE, known as the adjoint equation. This adjoint equation describes how the gradients of the loss function with respect to the latent state evolve backward in time. The gradients with respect to the SDE parameters ($f_{\theta}$ and $g_{\theta}$'s weights) can then be computed by integrating another auxiliary DE that accumulates these gradients along the backward pass. Therefore, the adjoint sensitivity method is used for memory-efficient computation of gradients through the SDE solver. 
\subsubsection*{Overall Backpropagation Flow}
\begin{enumerate}
	\item \textbf{Decoder Gradients:} Using AD, gradients are first computed for the parameters ($\theta$) of decoder based on the Gaussian NLL reconstruction loss, with respect to its inputs ($\mathbf{z}_t$ and $\mathbf{e}_{d}$). This provides the initial gradients $\partial L / \partial \mathbf{z}_t$ and $\partial L / \partial \mathbf{e}_{d}$.
	\item \textbf{SDE Adjoint Gradients:} The gradients $\partial L / \partial \mathbf{z}_t$ for each time point are then passed to the adjoint SDE solver. The adjoint method efficiently computes $\partial L / \partial \mathbf{z}_0$ (gradients with respect to the initial latent state) and $\partial L / \partial \theta_f$, $\partial L / \partial \theta_g$ (gradients with respect to the drift and diffusion network parameters) by integrating the adjoint system backward in time. This is crucial for memory efficiency, especially with long time series.
	\item \textbf{Encoder Gradients:} Finally, $\partial L / \partial \mathbf{z}_0$ is used to compute gradients for the parameters ($\phi$) of encoder with respect to $\boldsymbol{\mu}_{z_0}$ and $\log \boldsymbol{\sigma}^2_{z_0}$ through the reparameterization trick, which then backpropagate through the encoder network to update $\phi$.
	\item \textbf{District Embedding Gradients:} Gradients for the district embedding $\mathbf{e}_{d}$ are accumulated from the encoder, SDE, and decoder, and then used to update in the embedding layer.
\end{enumerate}	
This end-to-end backpropagation process, facilitated by adjoint methods for the SDE part, allows the entire VAE-SDE model to be trained effectively using gradient descent.
\section{Experimental Setup}
In this section, we will define and solve a socioeconomic dynamical problem utilizing the described method. 
\subsection{Problem Description}
The dataset consists of monthly observations over 168 months ($T = 168$) for six socioeconomic indicators across 30 districts in Odisha, India. These indicators include access to electricity, education attainment, housing availability, piped clean drinking water, clean cooking fuel, and improved sanitation facilities. In this regard, data was collected for year 2007 from District Level Household and Facility Survey (DLHS)-3 \cite{dlhs3_dataset} and for year 2015 and 2020 from National Family Health Survey (NFHS)-4 and NFHS-5 \cite{dhs_dataset} regarding the said socioeconomic indicators. \\

To make it a stochastic dynamical system, month-wise synthetic data between survey years were generated using a Brownian bridge interpolation scheme, which preserves the observed annual survey values while allowing stochastic fluctuations between them. This approach avoids unrealistic linear interpolation and reflects unobserved short-term socioeconomic variability.\\

A Brownian Bridge from $(t_a,x_a)$ to $(t_b,x_b)$ can be defined by Eq. (\ref{eq:brownian bridge}) \cite{ross1995stochastic}:
\begin{equation}\label{eq:brownian bridge}
	x_t = x_a + \frac{t - t_a}{t_b -t_a}(x_b-x_a) + \sigma \left[W(t-t_a)- \frac{t - t_a}{t_b -t_a}W(t_b-t_a)\right],
\end{equation}
where $W(t)$ be standard Brownian motion, $\sigma > 0$ controls the volatility.
\subsection{Computational Environment Settings}
This subsection will cover all the settings applied in the methodology, such as input, latent, and district embedding dimensions, probabilistic distribution for encoder and decoder, architecture details of all involved NNs, etc. \\

To begin with, the embedding layer consisting of 16 neurons ($m = 16$) is initialized for each district ID. There are six socioeconomic variables ($N = 6$) affecting the poverty status of households in a district, which include
\begin{enumerate}
	\item Percentage of households that have an electricity connection,
	\item Percentage of population obtained education at level 10 or more,
	\item Percentage of households that own a concrete-made house,
	\item Percentage of households that have a piped clean water connection in their yard,
	\item Percentage of households that use clean cooking fuel (gas facility in the kitchen),
	\item Percentage of households that have improved sanitation facilities (toilet with flush).
\end{enumerate}
The encoder takes 22 input neurons, which include six socioeconomic indicators and 16 district embedding vectors, two hidden layers with SiLU activation, and an output layer with linear activation of 8 neurons, which includes 4-dimensional (n = 4) latent vectors for mean and variance.\\ 

In the latent-NSDE, the drift $f_\theta$ and diffusion $g_\theta$ networks consider 20 input neurons (latent + district embedding), two hidden layers of 64 neurons with SiLU activation, and an output layer with four neurons. Linear activation is utilized in the drift network, and $\tanh$ activation is used in the diffusion network to control the magnitude of diffusion. \\

Finally, the decoder takes 20 input neurons, which include four latent dimension vectors and 16 district embedding vectors, two hidden layers with SiLU activation, and an output layer with linear activation of 12 neurons, which includes 6-dimensional socioeconomic dimension vectors for mean and variance.
\subsubsection*{Prior distribution}
The standard normal distribution is the most common choice for the prior distribution over the latent space, given by
\begin{equation}\label{eq:prior distribution}
	p(\mathbf{z}_0) = \mathcal{N}(\mathbf{z}_0; \mathbf{0}, \mathbf{I}),
\end{equation}
where $\mathbf{0}$ is a vector of zeros and $\mathbf{I}$ is the identity matrix, implying a mean of zero and unit variance for each independent latent dimension.
\subsubsection*{Negative log likelihood estimation}
The NLL from Eq. (\ref{eq:final loss analytic}) can be written as
\begin{equation}\label{eq:NLL General}
	-\mathbb{E}_{q_{\phi}(\mathbf{Z}|\mathbf{Y})} [\log p_\theta(\mathbf{Y}|\mathbf{Z})] = -\int\log p_\theta(\mathbf{Y}|\mathbf{Z}) q_{\phi}(\mathbf{Z}|\mathbf{Y})d\mathbf{Z}.
\end{equation}
This integral is intractable, therefore using Monte Carlo approximation we can write Eq. (\ref{eq:NLL General}) as 
\begin{equation}
	\mathbb{E}_{q_{\phi}(\mathbf{Z}|\mathbf{Y})} [-\log p_\theta(\mathbf{Y}|\mathbf{Z})] \approx \frac{1}{L}\sum_{l=0}^{L-1}-\log p_\theta(\mathbf{Y}|\mathbf{Z}^l),
\end{equation}
where $\mathbf{Z}^l \sim  q_{\phi}(\mathbf{Z}|\mathbf{Y})$. In practice, single unbiased Monte Carlo estimator is used, as a result final NLL looks likes 
\begin{equation}
	\begin{aligned}
		\mathbb{E}_{q_{\phi}(\mathbf{Z}|\mathbf{Y})} [-\log p_\theta(\mathbf{Y}|\mathbf{Z})] &\approx -\log p_\theta(\mathbf{Y}|\mathbf{Z})\\
		&= \frac{1}{2NT} \sum_{t=0}^{T-1}\sum_{j=1}^{N} \left[ \log(2\pi) + \log \sigma^2_{y_t,j} + \frac{(y_{t,j} - \mu_{y_t,j})^2}{\sigma^2_{y_t,j}} \right]
	\end{aligned}
\end{equation}
The constant term $\log(2\pi)$ is often omitted during optimization as it does not affect the gradients. Therefore, the practical NLL is simplifies to
\begin{equation}\label{eq:model nll}
	\mathbb{E}_{q_{\phi}(\mathbf{Z}|\mathbf{Y})} [-\log p_\theta(\mathbf{Y}|\mathbf{Z})] \approx
	 \frac{1}{2NT} \sum_{t=0}^{T-1}\sum_{j=1}^{N} \left[ \log \sigma^2_{y_t,j} + \frac{(y_{t,j} - \mu_{y_t,j})^2}{\sigma^2_{y_t,j}} \right],
\end{equation}
where $y_{t,j}$ is the actual observed value for feature $j$ at time $t$, $\mu_{y_t,j}$ is the predicted mean for feature $j$ at time $t$ by the decoder, and $\log \sigma^2_{y_t,j}$ is the predicted log-variance for feature $j$ at time $t$ by the decoder.\\

This loss function encourages the decoder to predict the correct mean and learn appropriate variances for uncertainty estimates in the predictions. If the model is very confident and incorrect, the NLL will be high. If it is uncertain and correct, the NLL will be lower.

\subsubsection{KL divergence term calculation}
In the VAE framework, the KL divergence term in the ELBO encourages the approximate posterior distribution $q_{\phi}(\mathbf{z}|\mathbf{x})$ to be close to a predefined prior distribution $p(\mathbf{z})$. Here in V-NSDE model, this KL divergence specifically applies to the initial latent state $\mathbf{z}_0$.\\

In our case, $q_{\phi}(\mathbf{z}_0 | \mathbf{y}_0, \mathbf{e}_{d})$ and $p(\mathbf{z}_0) $ follows distributions defined in Eq. (\ref{eq:posterior distribution phi}, \ref{eq:prior distribution}). Therefore, the closed form is given by
\begin{equation}\label{eq:model KL}
 D_{KL}(q_{\phi}(\mathbf{z}_0|\mathbf{y}_0, \mathbf{e}_{d}) || p(\mathbf{z}_0)) =\frac{1}{2n} \sum_{j=1}^{n} \left[ -\log \sigma^2_{z_0,j} - 1 + \sigma^2_{z_0,j} + \mu^2_{z_0,j} \right],
\end{equation}
where $\sigma^2_{z_0,j}$ denotes the $j$-th element of $\exp(\log \boldsymbol{\sigma}^2_{z_0})$.\\

In our specific V-NSDE model, the KL divergence term is calculated only for the initial latent state $\mathbf{z}_0$ because the subsequent evolution of $\mathbf{z}_t$ is governed by the SDE dynamics, which defines the continuous-time prior over the latent trajectory $p(\mathbf{Z})$. Therefore, the KL divergence term specifically focuses on $D_{KL}(q_{\phi}(\mathbf{z}_0|\mathbf{y}_0, \mathbf{e}_{d}) || p_\theta(\mathbf{z}_0))$. The reconstruction loss, however, spans the entire time series, comparing the predicted $\mathbf{y}_t$ distributions with the observed $\mathbf{y}_t$ for all $t$. The total ELBO loss considered in the model is given by Eqs. \eqref{eq:model nll}-\eqref{eq:model KL} as
\begin{equation}
	\begin{aligned}
		\mathcal{L}(\phi, \theta; \mathbf{Y}) &= \frac{1}{2NT} \sum_{t=0}^{T-1}\sum_{j=1}^{N} \left[ \log \sigma^2_{y_t,j} + \frac{(y_{t,j} - \mu_{y_t,j})^2}{\sigma^2_{y_t,j}} \right]\\
		 & \quad \quad + \beta . \frac{1}{2n} \sum_{j=1}^{n} \left[ -\log \sigma^2_{z_0,j} - 1 + \sigma^2_{z_0,j} + \mu^2_{z_0,j} \right].
	\end{aligned}
\end{equation}

In the backpropagation as described in Section \ref{sec. bacward model}, for the optimization propose we have considered Adam optimizer with default learning rate of $0.001$ \cite{kingma2017adam}. The model was trained for $1000$ epochs till the ELBO loss got stabilized. 
\section{Results and Discussion}
This section presents the results obtained from the trained V-NSDE model applied to the stochastic socioeconomic indicator data of Odisha. The model demonstrates stable and effective training behavior.

\subsection{Model Training and Convergence}

The V-NSDE model exhibits stable and well-behaved training dynamics. The ELBO loss decreases smoothly  within 1000 training epochs, indicating effective optimization and numerical stability. At the end of training, the final ELBO value $-10.746209$ is obtained, demonstrating a satisfactory balance between data reconstruction accuracy and latent space regularization.\\

The ELBO objective is primarily dominated by the Gaussian NLL term that reflects accurate reconstruction of the observed socioeconomic trajectories. The KL divergence term increases gradually during training due to the application of $\beta$-annealing and later gets stable. Which explains that it successfully prevented the posterior collapse and encouraged meaningful latent representations. As shown in Fig. (\ref{fig:loss-curve-nsde}), the overall convergence behavior confirms that the V-NSDE model is well-trained and suitable for capturing the stochastic dynamics present in district-level socioeconomic data.
\begin{figure}
	\centering
	\includegraphics[width=0.9\linewidth]{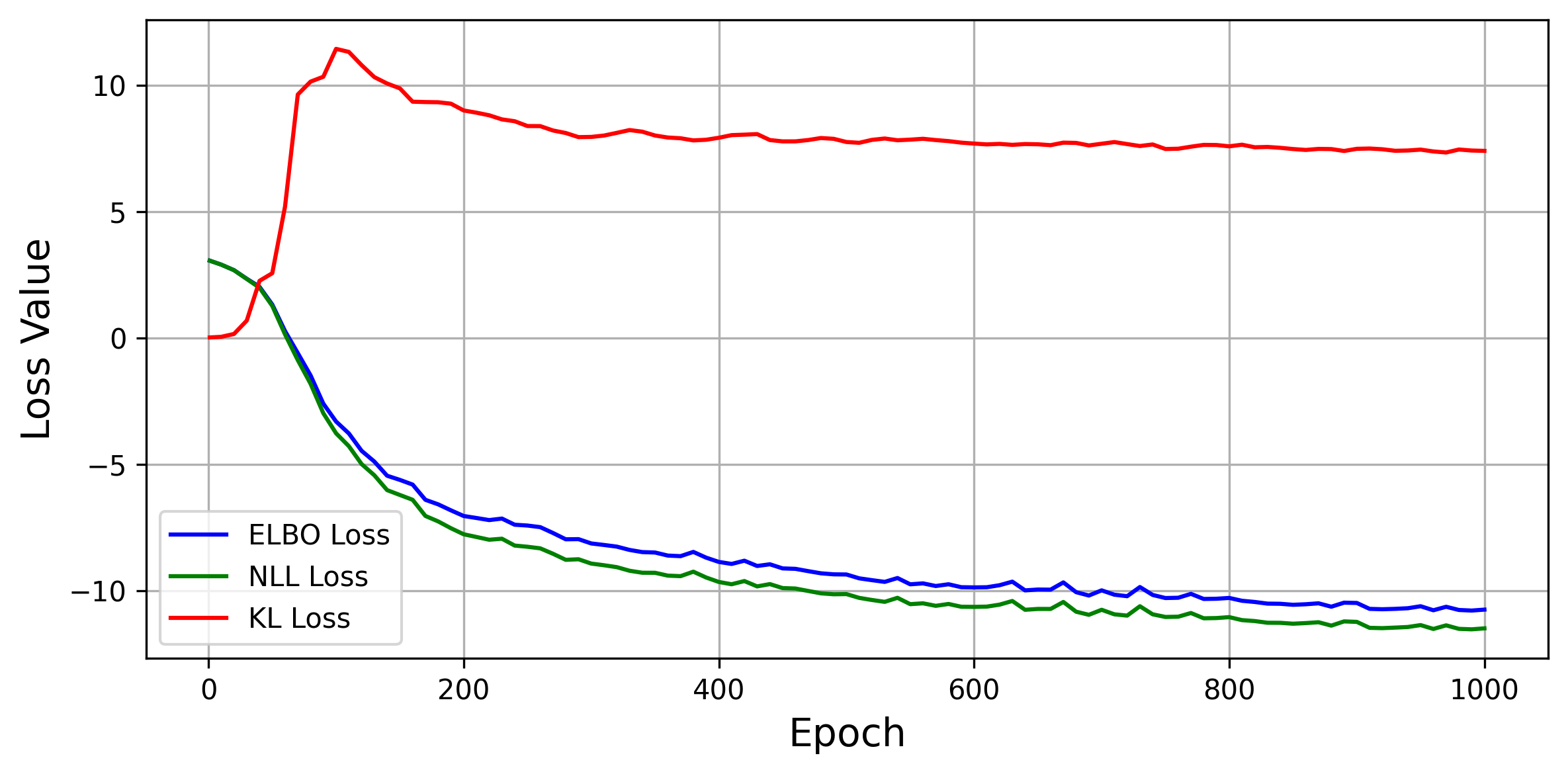}
	\caption{Training losses over epochs}
	\label{fig:loss-curve-nsde}
\end{figure}

\subsection{District-wise Quantitative Model Fit Using NLL}
To quantitatively assess model performance across districts, as depicted in Table \ref{tab:NLL performance}, the Gaussian NLL was computed separately for each of the 30 districts of Odisha. For each district, the trained V-NSDE model predicted the mean and log-variance of the time series. The NLL is evaluated by comparing the actual normalized observations against these probabilistic predictions obtained from the V-NSDE. These results highlight the heterogeneous predictive performance of the model across districts. These results also justify the necessity of district-specific latent embedding representations within a single V-NSDE modeling framework.

\begin{table}
	\caption{NLL performance across districts on trained model}
	\centering
	\begin{tabular}{|l c|l c|}
		\hline
		\textbf{District}&\textbf{ NLL} & \textbf{District}  & \textbf{NLL} \\
		\hline
		Rayagada & -12.599152&Debagarh& -11.589890\\ \hline
		Puri &-12.499263&Nayagarh& -11.562165\\ \hline
		Nuapada &-12.492243&Sonapur& -11.537984\\ \hline
		Bhadrak & -12.302701& Ganjam& -11.523965\\ \hline
		Cuttack& -12.268328&Balangir& -11.468771\\ \hline
		Malkangiri& -12.147593&Baleshwar& -11.392648\\ \hline
		Khandamal& -12.119408&Kendrapara& -11.343279\\ \hline
		Kendujhar& -11.975780&Kalahandi& -11.311354\\ \hline
		Nabarangapur& -11.891778&Jagatsinghpur& -11.011625\\ \hline
		Sundargarh& -11.889838& Gajapati& -10.979132\\ \hline
		Sambalpur& -11.875984&Baudh& -10.938087\\ \hline
		Koraput& -11.737288&Dhenkanal& -10.704522\\ \hline
		Bargarh& -11.737288&Jajapur& -10.615693\\ \hline
		Khordha &-11.681852& Angul& -10.251976\\ \hline
		Mayurbhanj& -11.663747&Jharsuguda& -10.069868\\ 
		\hline
		\multicolumn{4}{|c|}{ $\text{\textbf{Total ELBO}} = \text{Mean NLL} + \beta \times \text{Mean KL divergence}$} \\
		\hspace{0.5cm} -10.746209	& \multicolumn{3}{l|}{ $=-11.487712 +0.1 \times 7.415025$} \\
		\hline
	\end{tabular}
	\label{tab:NLL performance}
\end{table}
\subsection{Mean Trajectory Prediction and Temporal Trend Preservation}
Across all considered socioeconomic variables (education, housing, sanitation, drinking water access, clean cooking fuel, and electricity), the predicted mean trajectories show strong visual agreement with the observed district-level data. The V-NSDE model consistently preserves long-term temporal trends, accurately reflecting increasing or decreasing patterns over the study period. For visualization of such, the figures of three KBK (Koraput, Balangir, and Kalahandi) region districts are shown in Figures \ref{fig:balangirprediction}--\ref{fig:kalahandiprediction}.\\

The V-NSDE model successfully captures key turning points in the trajectories, which may be due to some reasons such as policy interventions, survey periods, or structural shifts. This demonstrates the ability of the stochastic latent dynamics to adapt to both smooth evolutions and volatility in the data trend.
\begin{figure}
	\centering
	\includegraphics[width=0.9\linewidth]{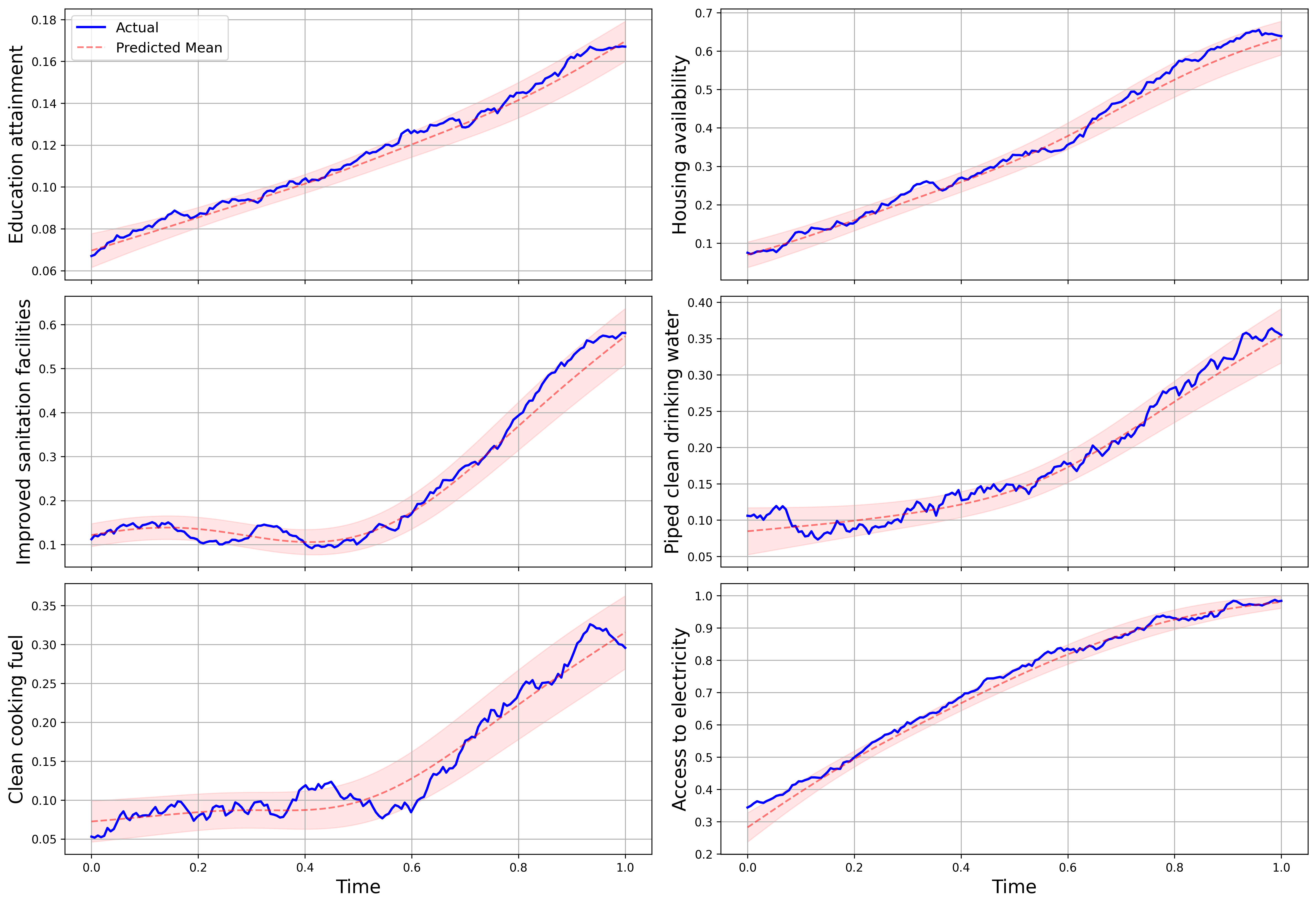}
	\caption{Likelihood predictions for Balangir District}
	\label{fig:balangirprediction}
\end{figure}
\begin{figure}
	\centering
	\includegraphics[width=0.9\linewidth]{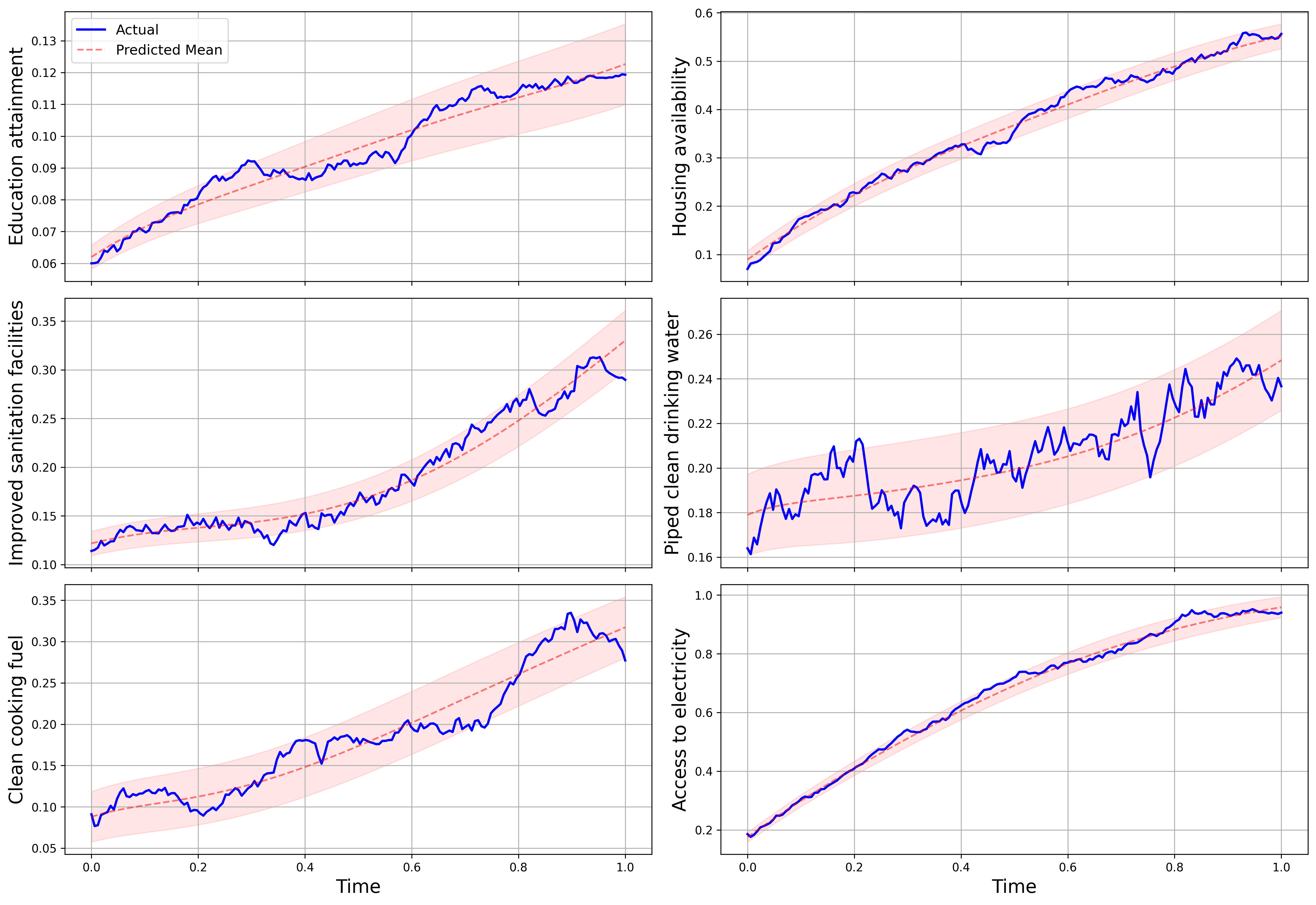}
	\caption{Likelihood predictions for Koraput District}
	\label{fig:koraputprediction}
\end{figure}
\begin{figure}
	\centering
	\includegraphics[width=0.9\linewidth]{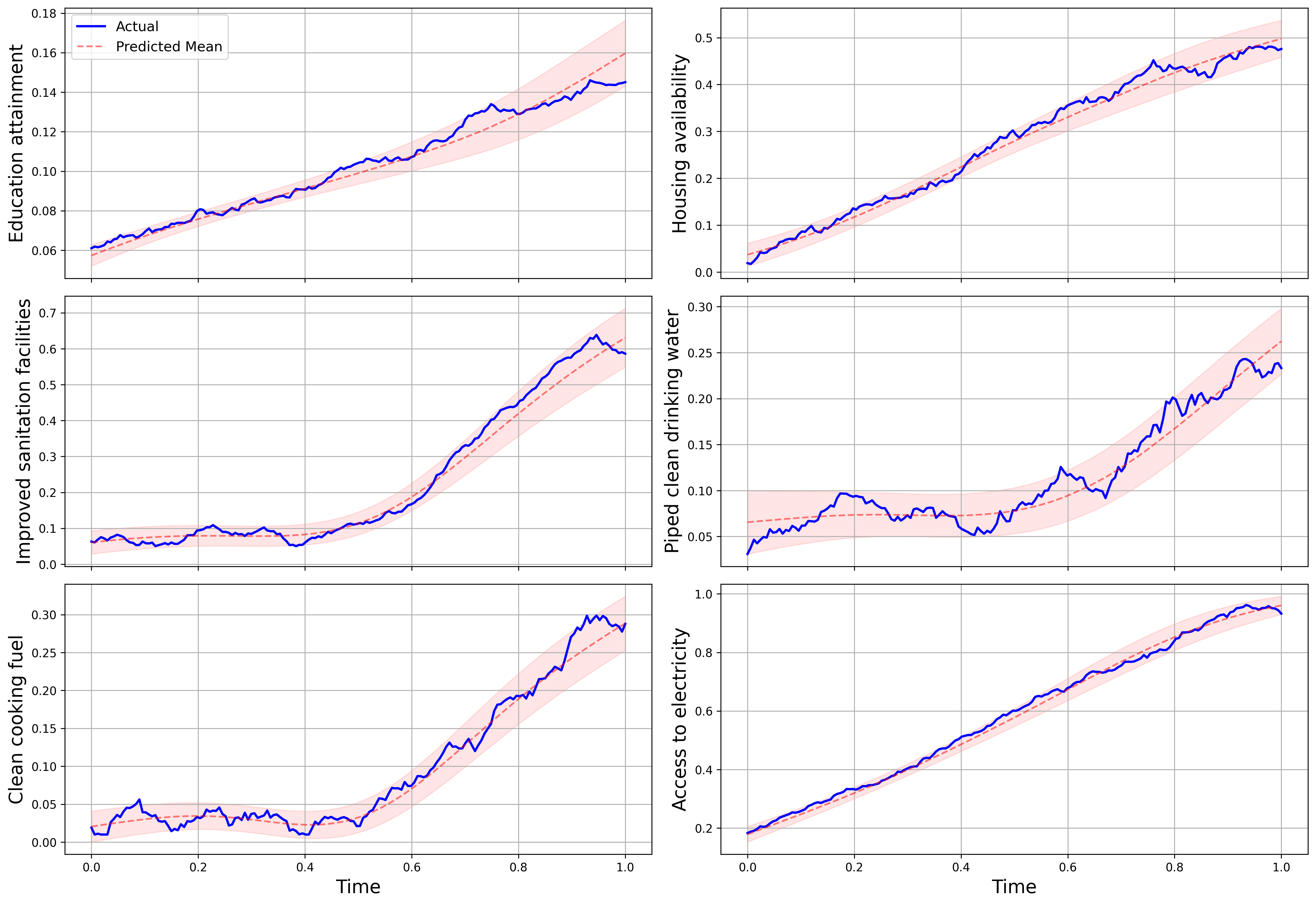}
	\caption{Likelihood predictions for Kalahandi District}
	\label{fig:kalahandiprediction}
\end{figure}

A major advantage of the proposed V-NSDE framework lies in its principled uncertainty quantification. The model outputs mean $\mathbf{\mu}$ and log-variance $\log \mathbf{\sigma^2}$ as shown by Eq. \eqref{eq:decoder}. Therefore, we have also predicted $\pm 2\sigma$ confidence intervals around the mean trajectories which effectively captures the inherent stochasticity and data uncertainty present in the socioeconomic indicators.\\

Interestingly, almost all observed data points lie within the predicted uncertainty bands, with only rare instances of observations falling outside these intervals. Also, wider predictive intervals are consistently observed during data-sparse periods and in phases exhibiting higher volatility. This pattern aligns with real-world expectations for irregularly sampled and noisy socioeconomic data. This behavior confirms that the model avoids overconfident predictions and provides realistic probabilistic forecasts.\\

The trained model also reveals clear district-level heterogeneity in both drift behavior and volatility structure. Different districts shows distinct temporal evolution patterns which reflects variations in socioeconomic development pathways. For example, backward and less-developed KBK region districts, display higher predictive uncertainty and increased volatility in drinking water, cooking fuel, and sanitation facilities indicator as depicted in Figures \ref{fig:balangirprediction}, \ref{fig:koraputprediction}, and \ref{fig:kalahandiprediction}.\\

These qualitative yet consistent pattern across districts shows that the V-NSDE is able to effectively learns district-specific characteristics through latent embeddings. It allows a single unified model to represent diverse regional dynamics without requiring separate district-wise models.\\

The empirical analysis yields several important insights:
\begin{enumerate}
	\item Education-related indicators exhibit relatively smoother temporal dynamics with lower predictive uncertainty across districts.
	\item Housing and sanitation indicators show higher volatility, reflecting sensitivity to infrastructural and policy-driven changes.
	\item Backward districts consistently display wider predictive intervals, indicating greater uncertainty and structural variability.
	\item The district-wise NLL analysis provides a clear quantitative measure of model fit, complementing the qualitative trajectory and uncertainty assessments.
\end{enumerate}

The results demonstrate that the proposed V-NSDE framework effectively captures mean dynamics, uncertainty structure, and district-level heterogeneity in socioeconomic time series. This makes the model a robust and interpretable tool for modeling complex and irregularly sampled socioeconomic processes.
\section{Conclusion}
In this study, we proposed a V-NSDE framework for modeling the temporal dynamics of multiple socioeconomic indicators related to household characteristics. The proposed approach formulates socioeconomic evolution as a continuous-time stochastic dynamical system, explicitly accounting for uncertainty arising from measurement errors, policy shocks, and unobserved heterogeneity.\\

The V-NSDE framework integrates VAEs to project high-dimensional, irregularly sampled socioeconomic data into a low-dimensional latent space, where the underlying dynamics are governed by a Neural Stochastic Differential Equation with neural network-parameterized drift and diffusion functions. This latent representation significantly reduces model complexity while preserving essential structural information. The latent SDE is solved using standard numerical SDE solvers, and the observed variables are reconstructed through a decoder network. Model training is carried out by maximizing the Evidence Lower Bound (ELBO), providing a fully probabilistic learning objective that balances reconstruction accuracy and latent regularization.\\

Empirical results using district-level data from Odisha, India, demonstrate that the proposed model effectively captures both the long-term trends and short-term stochastic fluctuations of key socioeconomic indicators. In particular, the predicted mean trajectories closely follow the observed temporal patterns across all districts, while the associated predictive uncertainty-quantified through the estimated standard deviation bands-successfully encloses the observed data. This indicates that the model not only achieves accurate point predictions but also provides meaningful uncertainty estimates. Notably, stable convergence was achieved within 250 training epochs, with the ELBO reaching values around $-11$, highlighting the computational efficiency and practical feasibility of the proposed framework.\\

Beyond the specific case study of Odisha, the proposed V-NSDE framework is general and transferable. It is well suited to a broad class of socioeconomic and development-related dynamical problems, particularly in settings where data are scarce, noisy, or irregularly sampled, as is common in household surveys and administrative records in developing regions. By enabling continuous-time modeling with uncertainty quantification, the approach supports more reliable policy analysis, forecasting, and scenario evaluation.\\

Despite its advantages, the current study has certain limitations. The model assumes time-invariant district embeddings and does not explicitly incorporate exogenous covariates such as policy interventions or macroeconomic shocks. Moreover, model performance may be sensitive to the choice of latent dimension and network architecture, which were selected empirically.\\

Future work will focus on extending the framework to include time-varying embeddings, exogenous control variables, and hierarchical spatial dependencies across districts. Additionally, integrating causal constraints and exploring policy-counterfactual simulations within the V-NSDE framework represent promising directions. These extensions will further enhance the applicability of the proposed methodology for evidence-based socioeconomic planning and decision-making.

\appendix
\section{Existence and Uniqueness of Latent NSDEs with District Embeddings}
\label{app:existence_uniqueness_latent_nsde}

\subsection{Probabilistic Setup}

Let $(\Omega, \mathcal{F}, \{\mathcal{F}_t\}_{t\ge0}, \mathbb{P})$ be a complete filtered probability space satisfying the usual conditions, and
\[
W_t = (W_t^{(1)}, \dots, W_t^{(m)})
\]
be an $m$-dimensional standard Brownian motion adapted to the filtration $\{\mathcal{F}_t\}$. All stochastic processes considered are assumed to be adapted to this filtration.
\subsection{Latent NSDE with District Embeddings}

We consider the latent NSDE as
\begin{equation}
	\label{eq:latent_nsde}	
	\begin{aligned}
		dZ_t^{(d)} &= f_\theta\!\left(Z_t^{(d)}, t, e_d\right)\,dt
		+ g_\theta\!\left(Z_t^{(d)}, t, e_d\right)\,dW_t^{(d)}, \\
		Z_0^{(d)} &\sim q_\phi(Z_0 \mid Y_0^{(d)}, e_d),
	\end{aligned}
\end{equation}
where $Z_t^{(d)} \in \mathbb{R}^k$ is the latent state for district $d$, $e_d \in \mathbb{R}^r$ is a fixed district embedding vector, $f_\theta:\mathbb{R}^k \times [0,T] \times \mathbb{R}^r \to \mathbb{R}^k$ is the neural drift, $g_\theta:\mathbb{R}^k \times [0,T] \times \mathbb{R}^r \to \mathbb{R}^{k\times m}$ is the neural diffusion, $W_t^{(d)}$ is an $m$-dimensional Brownian motion, and $q_\phi$ is an encoder-defined initial latent distribution.

\subsection{Assumptions}

\paragraph{Assumption A (Embedding Regularity).}
District embeddings are bounded, i.e.
\begin{equation}
	\label{eq:embedding_bound}
	\|e_d\| \le M_e < \infty, \quad \forall d.
\end{equation}

\paragraph{Assumption B (Measurability and Adaptedness).}
The functions $f_\theta(z,t,e_d)$ and $g_\theta(z,t,e_d)$ are Borel measurable in $(\mathbb{R}^k; [0,T];\mathbb{R}^r)$, progressively measurable with respect to $\{\mathcal{F}_t\}$, and continuous in $\mathbb{R}^k$ for almost every $t$.

\paragraph{Assumption C (Global Lipschitz Continuity).}
There exists a constant $L>0$ such that for all $z_1,z_2 \in \mathbb{R}^k$, all $t\in[0,T]$, and all embeddings $e_d$,
\begin{equation}
	\label{eq:lipschitz}
	\|f_\theta(z_1,t,e_d) - f_\theta(z_2,t,e_d)\| + \|g_\theta(z_1,t,e_d) - g_\theta(z_2,t,e_d)\| \le L \|z_1 - z_2\|.
\end{equation}
\paragraph{Assumption D (Linear Growth Condition).}
There exists $C>0$ such that
\begin{equation}
	\label{eq:linear_growth}
	\|f_\theta(z,t,e_d)\|^2 + \|g_\theta(z,t,e_d)\|^2
	\le C\big(1 + \|z\|^2 + \|e_d\|^2\big),
	\quad \forall z,t,d.
\end{equation}
\paragraph{Assumption E (Finite Second Moment of Initial Condition).}
The initial latent variable satisfies
\begin{equation}
	\label{eq:initial_condition}
	\mathbb{E}\|Z_0^{(d)}\|^2 < \infty.
\end{equation}

\subsubsection{Existence and Uniqueness Theorem}

\begin{theorem}[Existence and Uniqueness of Latent NSDE]
	\label{thm:latent_nsde_existence}
	Under Assumptions A--E, for each district $d$, the latent  NSDE in Eq. \eqref{eq:latent_nsde} admits a \textbf{unique strong solution}
	\[
	\{Z_t^{(d)}\}_{t\in[0,T]}
	\]
	such that $Z_t^{(d)}$ is $\mathcal{F}_t$-adapted and almost surely continuous and $\mathbb{E}\left[\sup_{t\in[0,T]} \|Z_t^{(d)}\|^2\right] < \infty$.
\end{theorem}
\begin{proof}
In this proof we will follows standard arguments for It\^{o} SDEs with parameterized coefficients. Therefore, Eq. \eqref{eq:latent_nsde} is equivalent to the following It\^{o} integral
\begin{equation}
	\label{eq:integral_form}
	Z_t^{(d)} = Z_0^{(d)}
	+ \int_0^t f_\theta(Z_s^{(d)}, s, e_d)\,ds
	+ \int_0^t g_\theta(Z_s^{(d)}, s, e_d)\,dW_s^{(d)}.
\end{equation}
Using Picards iteration, we define the sequence $\{Z^{(n)}\}_{n\ge0}$ as
\begin{equation}
	\label{eq:picard}
	\begin{aligned}
		Z_t^{(0)} &= Z_0^{(d)}, \\
		Z_t^{(n+1)} &= Z_0^{(d)}
		+ \int_0^t f_\theta(Z_s^{(n)}, s, e_d)\,ds
		+ \int_0^t g_\theta(Z_s^{(n)}, s, e_d)\,dW_s^{(d)}.
	\end{aligned}
\end{equation}
To prove existence of solution, we will show that the Picard iteration is a Cauchy in
$L^2\big(\Omega; C([0,T];\mathbb{R}^k)\big)$.
\paragraph{Existence of strong solution.} From the Picard iteration, we have
\[
\begin{aligned}
	Z_t^{(n+1)} - Z_t^{(n)}
	&=
	\int_0^t \!\big[f_\theta(Z_s^{(n)},s,e_d)-f_\theta(Z_s^{(n-1)},s,e_d)\big] ds \\
	&\quad + \int_0^t \!\big[g_\theta(Z_s^{(n)},s,e_d)-g_\theta(Z_s^{(n-1)},s,e_d)\big] dW_s.
\end{aligned}
\]
Taking supremum over $t\in[0,T]$, expectation and using $\|a+b\|^2 \le 2\|a\|^2 + 2\|b\|^2$, we obtain
\[ 
\begin{aligned}
\mathbb{E}\left[\sup_{t\le T}
\left\| Z_t^{(n+1)} - Z_t^{(n)} \right\|^2 \right]
 &=  2 \mathbb{E}\left[\sup_{t\le T}
\left\| \int_0^t \!\big[f_\theta(Z_s^{(n)},s,e_d)-f_\theta(Z_s^{(n-1)},s,e_d)\big] ds \right\|^2  \right]  \\
&\quad +  2 \mathbb{E}\left[\sup_{t\le T}
\left\| \int_0^t \!\big[g_\theta(Z_s^{(n)},s,e_d)-g_\theta(Z_s^{(n-1)},s,e_d)\big] dW_s \right\|^2 \right] .
\end{aligned}
\]
This can also be written as
\begin{equation}\label{eq: expectation equality}
	\mathbb{E}\left[\sup_{t\le T}
	\left\| Z_t^{(n+1)} - Z_t^{(n)} \right\|^2 \right]
	=  2  \mathbb{E}\left[
	\sup_{t\le T}
	\left\|
	\int_0^t \Delta f_s ds
	\right\|^2
	\right] \\
	+ 2\mathbb{E}\left[
	\sup_{t\le T}
	\left\|
	\int_0^t \Delta g_s dW_s
	\right\|^2
	\right] .
\end{equation}
where 
\[\Delta f_s  = f_\theta(Z_s^{(n)},s,e_d)-f_\theta(Z_s^{(n-1)},s,e_d), \quad \Delta g_s  = g_\theta(Z_s^{(n)},s,e_d)-g_\theta(Z_s^{(n-1)},s,e_d).\]

For drift term integral, using Cauchy-Schwarz inequality we get
\[
\mathbb{E}\left[
\sup_{t\le T}
\left\|
\int_0^t \Delta f_s ds
\right\|^2
\right]
\le
T .\mathbb{E} \left[\sup_{t\le T}\int_0^t \|\Delta f_s\|^2 ds\right].
\]

For the stochastic integral, the Itô Isometry inequality yields
\[
\mathbb{E}\left[
\sup_{t\le T}
\left\|
\int_0^t \Delta g_s dW_s
\right\|^2
\right]
\le  \mathbb{E}
\left[\sup_{t\le T}\int_0^t \|\Delta g_s\|^2 ds \right]\].

Therefore Eq. \eqref{eq: expectation equality} becomes

\begin{equation}
	\begin{aligned}
		\mathbb{E}\left[\sup_{t\le T}
		\left\| Z_t^{(n+1)} - Z_t^{(n)} \right\|^2 \right]
		&\le 2 T .\mathbb{E} \left[\sup_{t\le T}\int_0^t \|\Delta f_s\|^2 ds\right] + 2 \mathbb{E}
		\left[\sup_{t\le T}\int_0^t \|\Delta g_s\|^2 ds \right]\\
		&\le 2(T+1)\left(\mathbb{E} \left[\sup_{t\le T}\int_0^t \|\Delta f_s\|^2 ds + \sup_{t\le T}\int_0^t \|\Delta g_s\|^2 ds \right] \right)\\
		&=\ 2(T+1)\left(\mathbb{E} \left[\sup_{t\le T}\int_0^t (\|\Delta f_s\|^2 + \|\Delta g_s\|^2) ds \right] \right).
	\end{aligned}
\end{equation}
Using the global Lipschitz continuity of Eq. \eqref{eq:lipschitz}, we obtain
\begin{equation}
	\begin{aligned}
		\mathbb{E}\left[\sup_{t\le T}
		\left\| Z_t^{(n+1)} - Z_t^{(n)} \right\|^2 \right] 
		&\le 2L^2(T+1) \mathbb{E} \left[\sup_{t\le T}\int_0^t \left\|Z_s^{(n)} - Z_s^{(n-1)}\right\|^2 ds\right]\\
		&\le 2L^2(T+1) \mathbb{E} \left[\int_0^T \left\|Z_s^{(n)} - Z_s^{(n-1)}\right\|^2 ds\right]\\
		&= 2L^2(T+1) \int_0^T \mathbb{E} \left[\left\|Z_s^{(n)} - Z_s^{(n-1)}\right\|^2 \right] ds\\
		&\le 2L^2T(T+1)\mathbb{E}\left[\sup_{t\le T}\left\|Z_t^{(n)} - Z_t^{(n-1)}\right\|^2\right].
	\end{aligned}
\end{equation}
Choosing $T>0$ such that $\alpha = 2L^2T(T+1)<1$ gives
\begin{equation}
	\mathbb{E}\left[\sup_{t\le T}\|Z_t^{(n+1)} - Z_t^{(n)}\|^2\right]
	\le \alpha
	\mathbb{E}\left[\sup_{t\le T}\|Z_t^{(n)} - Z_t^{(n-1)}\|^2\right]
\end{equation}
Hence, the Picard iteration is a contraction, i.e. $\{Z^{(n)}\}$ is Cauchy in $L^2(\Omega)$. Therefore, there exists a unique fixed point which is the strong solution of the Eq. \eqref{eq:latent_nsde}. 

\paragraph{Convergence and Uniqueness.} 
We will show that the latent Neural SDE admits at most one strong solution. We will prove this by contradiction. Suppose that there exist two solutions $\{Z_t\}_{t\in[0,T]}$ and $\{\widetilde{Z}_t\}_{t\in[0,T]}$ of Eq. \eqref{eq:latent_nsde} with the same initial condition $Z_0$ and driven by the same Brownian motion.\\\\
The difference process is defined as
\begin{equation}
		D_t := Z_t - \widetilde{Z}_t =
		\int_0^t \Big( f_\theta(Z_s,s,e_d) - f_\theta(\widetilde{Z}_s,s,e_d) \Big)\,ds \\
		+
		\int_0^t \Big( g_\theta(Z_s,s,e_d) - g_\theta(\widetilde{Z}_s,s,e_d) \Big)\,dW_s.
\end{equation}

Taking the supremum over $t\in[0,T]$ and expectation, and applying the inequality
$\|a+b\|^2 \le 2\|a\|^2 + 2\|b\|^2$, we obtain
\begin{equation}
	\mathbb{E}\left[
	\sup_{t\le T} \|D_t\|^2
	\right]
	\le
	2\,\mathbb{E}\left[
	\sup_{t\le T}
	\left\|
	\int_0^t \Delta f_s\,ds
	\right\|^2
	\right]
	+
	2\,\mathbb{E}\left[
	\sup_{t\le T}
	\left\|
	\int_0^t \Delta g_s\,dW_s
	\right\|^2
	\right],
\end{equation}
where
\[
\Delta f_s = f_\theta(Z_s,s,e_d) - f_\theta(\widetilde{Z}_s,s,e_d),
\quad
\Delta g_s = g_\theta(Z_s,s,e_d) - g_\theta(\widetilde{Z}_s,s,e_d).
\]

Using the global Lipschitz continuity of $f_\theta$ and $g_\theta$ from Eq. \eqref{eq:lipschitz}, together with standard bounds for deterministic and stochastic integrals, we obtain
\begin{equation}
	\mathbb{E}\left[
	\sup_{t\le T} \|D_t\|^2
	\right]
	\le
	C
	\int_0^T
	\mathbb{E}\left[
	\sup_{u\le s} \|D_u\|^2
	\right] ds,
\end{equation}
for some constant $C>0$.\\

The Gr\"onwall's inequality states that if $g(t) \leq A + C \int_0^t g(s) ds$ and if $A=0$, then $g(t) = 0$ for all $t \in [0, T]$. It follows that
\[
\mathbb{E}\left[
\sup_{t\le T} \|D_t\|^2
\right] = 0.
\]

Therefore, our assumption was wrong. Hence,
\[
Z_t = \widetilde{Z}_t
\quad \text{for all } t\in[0,T].
\]

Hence it is proved that the strong solution of Eq. \eqref{eq:latent_nsde} is unique.
\end{proof}
\bibliographystyle{ieeetr}
\bibliography{ref}
\end{document}